\pdfoutput=1

\documentclass[11pt]{article}

\usepackage[preprint]{acl}

\usepackage{times}
\usepackage{latexsym}

\usepackage[T1]{fontenc}

\usepackage[utf8]{inputenc}
\usepackage{microtype}
\usepackage{inconsolata}

\usepackage{amsmath,amsfonts,bm}

\def\eqref#1{equation~\ref{#1}}

\def\1{\bm{1}}

\def\vs{{\bm{s}}}

\def\vx{{\bm{x}}}

\def\mA{{\bm{A}}}

\def\mM{{\bm{M}}}

\def\mX{{\bm{X}}}

\DeclareMathAlphabet{\mathsfit}{\encodingdefault}{\sfdefault}{m}{sl}
\SetMathAlphabet{\mathsfit}{bold}{\encodingdefault}{\sfdefault}{bx}{n}

\def\gE{{\mathcal{E}}}

\def\gG{{\mathcal{G}}}

\def\gV{{\mathcal{V}}}

\def\sR{{\mathbb{R}}}

\def\emA{{A}}

\usepackage{wrapfig}
\usepackage{subfigure}
\usepackage{graphicx}
\usepackage{bbm}
\usepackage{bbold}
\usepackage{listings}
\usepackage{algorithm} 
\usepackage{algpseudocode} 
\usepackage{xspace}
\usepackage{subcaption}

\usepackage{amsthm}
\definecolor{codegreen}{rgb}{0,0.6,0}
\definecolor{codegray}{rgb}{0.5,0.5,0.5}
\definecolor{codepurple}{rgb}{0.58,0,0.82}
\definecolor{backcolour}{rgb}{0.95,0.95,0.92}

\newcommand{\datasetFont}{\text}
\newcommand{\ours}{\datasetFont{LLMExplainer}\xspace}

\newcommand{\mutag}{\datasetFont{MUTAG}\xspace}

\newcommand{\fluc}{\datasetFont{Fluoride-Carbonyl}\xspace}
\newcommand{\alca}{\datasetFont{Alkane-Carbonyl}\xspace}

\newcommand{\bamv}{\datasetFont{BA-Motif-Volume}\xspace}
\newcommand{\bamc}{\datasetFont{BA-Motif-Counting}\xspace}
\newtheorem{theorem}{Theorem}
\newtheorem{problem}{Problem}
\newtheorem{hypothesis}{Hypothesis}

\theoremstyle{definition}
\newtheorem{definition}{Definition}[section]

\title{\ours: Large Language Model based Bayesian Inference for Graph Explanation Generation}

\author{
 \textbf{Jiaxing Zhang$^*$\textsuperscript{1}},
 \textbf{Jiayi Liu$^*$\textsuperscript{2}},
 \textbf{Dongsheng Luo\textsuperscript{3}},
 \textbf{Jennifer Neville\textsuperscript{2 4}},
 \textbf{Hua Wei\textsuperscript{5}}
\\
\small{
 \textsuperscript{1}New Jersey Insititute of Technology,
 \textsuperscript{2}Purdue University,
 \textsuperscript{3}Florida International University,
 \textsuperscript{4}Microsoft Research,
 \textsuperscript{5}Arizona State University
 }
\\
 \small{
    jz48@njit.edu,
    liu2861@purdue.edu,
    dluo@fiu.edu,
    neville@purdue.edu,
    hua.wei@asu.edu
 }
}

\begin{document}
\maketitle
\def\thefootnote{*}\footnotetext{These authors contributed equally to this work.}
\begin{abstract}
Recent studies seek to provide Graph Neural Network (GNN) interpretability via multiple unsupervised learning models. Due to the scarcity of datasets, current methods easily suffer from learning bias. To solve this problem, we embed a Large Language Model (LLM) as knowledge into the GNN explanation network to avoid the learning bias problem. We inject LLM as a Bayesian Inference (BI) module to mitigate learning bias. The efficacy of the BI module has been proven both theoretically and experimentally. We conduct experiments on both synthetic and real-world datasets. The innovation of our work lies in two parts: 1. We provide a novel view of the possibility of an LLM functioning as a Bayesian inference to improve the performance of existing algorithms; 2. We are the first to discuss the learning bias issues in the GNN explanation problem.
\end{abstract}

\section{Introduction}
\label{sec:intro}

Interpreting the decisions made by Graph Neural Networks (GNNs)~\citep{scarselli09gnnmodel} is crucial for understanding their underlying mechanisms and ensuring their reliability in various applications. As the application of GNNs expands to encompass graph tasks in social networks~\citep{fan19social, min21social}, molecular structures~\citep{chereda2019utilizing, mansimov19deepggnn}, traffic flows~\citep{wang20traffic, Li_Zhu_2021_traffic, wu19graphwave, ijcai2018p0505}, and knowledge graphs~\citep{sorokin-gurevych-2018-modeling-knowledgegraph}, 
GNNs achieve state-of-the-art performance in tasks including node classification, graph classification, graph regression, and link prediction. 
The burgeoning demand highlights the necessity of enhancing GNN interpretability to strengthen model transparency and user trust, particularly in high-stakes settings~\citep{yuan2022explainability,longa2022explaining}, and to facilitate insight extraction in complex fields such as healthcare and drug discovery~\citep{zhang2022trustworthy,wu2022survey,li2022survey}.

\begin{figure}
\begin{center}
     \includegraphics[width=0.4\textwidth]{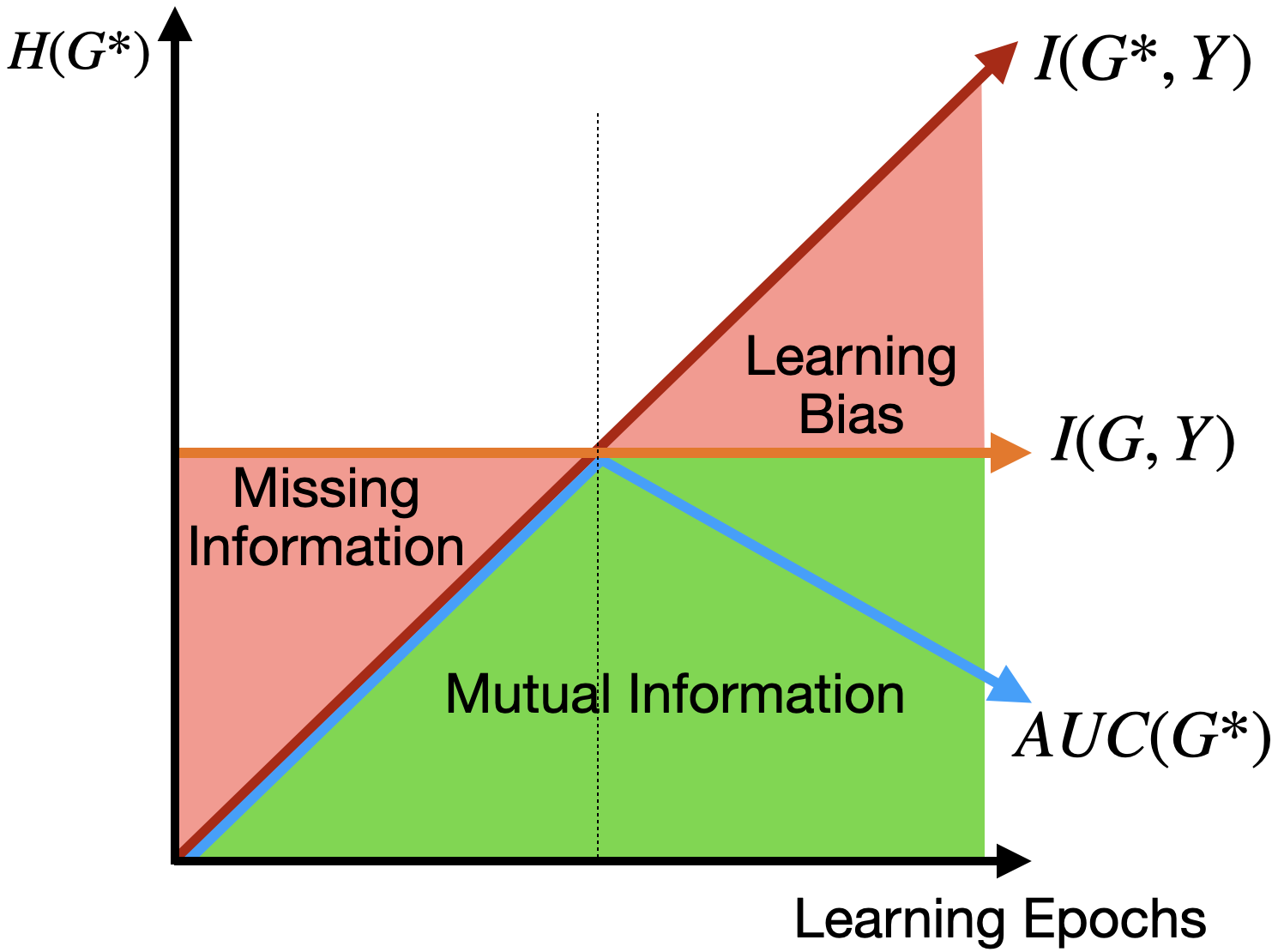}
    \caption{Intuitive visualization of learning bias.}
    \label{fig: intui}
\end{center}
\end{figure}

Recent efforts in explaining GNN~\citep{ying2019gnnexplainer, luo2020parameterized, zhang2023mixupexplainer} have sought to enhance GNN interpretability through multiple learning objectives, with a particular focus on the graph information bottleneck (GIB) method. GIB's goal is to distill essential information from graphs for clearer model explanations. However, the effectiveness of GIB hinges on the availability of well-annotated datasets, which are instrumental in accurately training and validating these models. Unfortunately, such datasets are rare, primarily due to the significant expert effort required for accurate annotation and occasionally due to the inherent complexity of the graph data itself. This scarcity poses a serious challenge, leading to a risk of learning bias in explaining GNN. Learning bias arises when the model overly relies on the limited available data, potentially leading to incorrect or over-fitted interpretations. We illustrate this phenomenon in Fig.~(\ref{fig: intui}) and provide empirical evidence in Fig.~(\ref{fig: RQ2}). 

As demonstrated in the figures, learning bias becomes increasingly problematic as the model continues to train on sparse data. Initially, the model might improve as it learns to correlate the sub-graph $G^*$ with the label $Y$, optimizing mutual information $I(G^*, Y)$. However, beyond a certain point, keeping optimizing $I(G^*, Y)$ leads to an over-reliance on the limited and possibly non-representative data available, thereby aggravating the learning bias. This situation is depicted through the divergence of mutual information and actual performance metrics, such as AUC, where despite higher mutual information, the practical interpretability and accuracy of the model decline. 

To mitigate the learning bias, current models often stop training early, a practice to prevent the exacerbation of learning bias. However, this approach is inherently flawed, especially in real-world applications lacking comprehensive validation datasets, leading potentially to under-fitting and inadequate model generalization. This emphasizes the need for innovative approaches to model training and interpretation that can navigate the challenges posed by sparse ground truth in explaining GNNs.

To address the challenges posed by sparse ground truth annotations and the consequent risk of learning bias in explaining GNNs, we propose \ours, a versatile GNN explanation framework that incorporates the insights from Large Language Model (LLM) into a wide array of backbone GNN explanation models, ranging from instance-level to model-level explanation models\citep{rgexp, ying2019gnnexplainer, luo2020parameterized, subgraphx, spinelli2022meta, wang2021causal, subgraphx, rgexp, chen2024generatingindistributionproxygraphs}. The LLMs act as a grader, and the evaluations from LLMs are then integrated into the model to inform a weighted gradient descent process. Specifically, to ensure a satisfactory level of explanation performance, we embed Bayesian Variational Inference into the original GNN explainers and use LLM as the prior knowledge in Bayesian Variational Inference. We prove that with the injection of LLM, \ours will mitigate the learning bias problem. Our experimental results show the effectiveness of enhancing the backbone explanation models with faster convergence and fortifying them against learning bias.

In summary, the main contributions of this paper are: 
\begin{itemize}
    \item We propose a new and general framework, \ours, which solves the problem of learning bias in the graph explanation process by embedding the Large Language Model into the graph explainer with a Bayesian inference process and improving the explanation accuracy. 
    \item We theoretically prove the effectiveness of the proposed algorithm and show that the lower bound of \ours is no less than the original bound of baselines. Our proposed method achieved the best performance through five datasets compared to the baselines.
    \item We are the first to discuss this learning bias problem in the domain of graph explanation and provide the potential of the Large Language Model as a Bayesian inference module to benefit the current works.
\end{itemize}

\section{Related Work}
\label{sec:relatedwork}

\subsection{Graph Neural Networks and Graph Explanations}
Graph neural networks (GNNs) are on the rise for analyzing graph structure data, as seen in recent research studies~\citep{dai2022towards, fan19social, hamilton2017inductive}. There are two main types of GNNs: spectral-based approaches~\citep{bruna2013spectral, kipf2016semi, tang2019chebnet} and spatial-based approaches~\citep{atwood2016diffusion, duvenaud2015convolutional, xiao2021learning}. Despite the differences, message passing is a common framework for both, using pattern extraction and message interaction between layers to update node embeddings. However, GNNs are still considered a black box model with a hard-to-understand mechanism, particularly for graph data, which is harder to interpret compared to image data. To fully utilize GNNs, especially in high-risk applications, it is crucial to develop methods for understanding how they work.

Many attempts have been made to interpret GNN models and explain their predictions~\citep{rgexp, ying2019gnnexplainer, luo2020parameterized, subgraphx, spinelli2022meta, wang2021causal}. These methods can be grouped into two categories based on granularity: (1) instance-level explanation, which explains the prediction for each instance by identifying significant substructures~\citep{ying2019gnnexplainer, subgraphx, rgexp}, and (2) model-level explanation, which seeks to understand the global decision rules captured by the GNN~\citep{luo2020parameterized, spinelli2022meta, bald19expgcn}. From a methodological perspective, existing methods can be classified as (1) self-explainable GNNs~\citep{bald19expgcn, dai21towards}, where the GNN can provide both predictions and explanations; and (2) post-hoc explanations~\citep{ying2019gnnexplainer, luo2020parameterized, subgraphx}, which use another model or strategy to explain the target GNN. In this work, we focus on post-hoc instance-level explanations, which involve identifying instance-wise critical substructures to explain the prediction. Various strategies have been explored, including gradient signals, perturbed predictions, and decomposition.

Perturbed prediction-based methods are the most widely used in post-hoc instance-level explanations. The idea is to learn a perturbation mask that filters out non-important connections and identifies dominant substructures while preserving the original predictions. For example, GNNExplainer~\citep{ying2019gnnexplainer} uses end-to-end learned soft masks on node attributes and graph structure, while PGExplainer~\citep{luo2020parameterized} incorporates a graph generator to incorporate global information. RG-Explainer~\citep{rgexp} uses reinforcement learning technology with starting point selection to find important substructures for the explanation.
\subsection{Bayesian Inference}
\citet{mackay1992bayesian} came up with the Bayesian Inference in general models, while \citet{graves2011practical} first applied Bayesian Inference to neural networks. Currently, Bayesian Inference has been applied broadly in computer vision (CV), nature language processing (NLP), etc~\citep{muller2021transformers, gal2015bayesian, xue2021bayesian, song2024imprint}. 

Bayesian Variational techniques have seen extensive uptake in Bayesian approximate inference. They adeptly reframe the posterior inference challenge as an optimization endeavor~\citep{wang2023scalable}. When compared to Markov Chain Monte Carlo, which is another Bayesian Inference method, Variational Inference exhibits enhanced convergence and scalability, making it better suited for tackling large-scale approximate inference tasks.

Due to the nature of Bayesian Variational Inference, it has been embedded into neural networks called Bayesian Neural Networks \citep{graves2011practical}. A major drawback of current Deep Neural Networks is that they use fixed parameter values, and fail to provide uncertainty estimations, resulting in a limitation in uncertainty. 
BNNs are extensively used in fields like active learning, Bayesian optimization, and bandit problems, as well as in out-of-distribution sample detection problems like anomaly detection and adversarial sample detection.
\subsection{Large Language Model}
Large Language Models (LLMs) have been widely used since 2023 \citep{bubeck2023sparks, brown2020language, zhou2022large}. Based on the architecture of Transformer\citep{vaswani2017attention}, LLMs have achieved remarkable success in various Natural Language Processing (NLP) tasks. LLMs have spurred discussions from multiple angles, including LLM efficiency \citep{liu2024cliqueparcel, wan2024tnt}, personalized LLMs \citep{mysore2023pearl, fang2024llm}, prompt engineering\citep{wei2022chain, song23c_interspeech}, fine tuning\citep{lai2024adaptive}, etc.

Beyond their traditional domain of NLP, LLMs have found extensive usage in diverse fields such as computer vision \citep{wang2024visionllm, dang2024realtime}, graph learning \citep{he2023harnessing}, and recommendation systems \citep{jin2023time, wu2024switchtab}, etc. By embedding LLMs into existing systems, researchers and practitioners have observed enhanced performance across various domains, underscoring the transformative impact of these models on modern AI applications.
\section{Preliminary}
\subsection{Notations and Problem Definition}
We summarize all the important notations in Table~\ref{app:tab:notation} in the appendix. We denote a graph as $G= (\mathcal{V}, \mathcal{E}; \mX, \mA)$, where $\mathcal{V} = \{v_1, v_2, ..., v_n\}$ represents a set of $n$ nodes and $ \mathcal{E} \subseteq \mathcal{V} \times \mathcal{V}$ represents the edge set. Each graph has a feature matrix $ \mX \in \sR^{n\times d} $ for the nodes. where in $\mX $, $ \vx_i  \in \sR^{1\times d} $ is the $d$-dimensional node feature of node $v_i$. $\mathcal{E}$ is described by an adjacency matrix $ \mA \in \{0,1\}^{n\times n}$. $\emA_{ij} = 1$ means that there is an edge between node $v_i$ and $v_j$; otherwise, $\emA_{ij} = 0$. 

For the graph classification task, each graph $G$ has a ground-truth label $\Bar{Y} \in \mathcal{C}$, with a GNN model $f$ trained to classify $G$ into its class, i.e., $f:(\mX, \mA) \mapsto \Bar{Y} \in \{1, 2, ..., C\}$. For the node classification task, each graph $G$ denotes a $K$-hop sub-graph centered around node $v_i$, with a GNN model $f$ trained to predict the label for node $v_i$ based on the node representation of $v_i$ learned from $G$. Since the node classification can be converted to computation graph classification task~\citep{ying2019gnnexplainer, luo2020parameterized}, we focus on the graph classification task in this work.
For the graph regression task, each graph $G$ has a label $\Bar{Y} \in \sR$, with a GNN model $f$ trained to predict $G$ into a regression value, i.e., $f:(\mX, \mA) \mapsto \Bar{Y} \in \sR$.
 
Informative feature selection has been well studied in non-graph structured data~\citep{Li17featureselect},  and traditional methods, such as concrete autoencoder~\citep{balin2019concrete}, can be directly extended to explain features in GNNs. In this paper, we focus on discovering important typologies. Formally, the obtained explanation $G^{*}$ is depicted by a binary mask $\mM \in \{0, 1\}^{n\times n}$ on the adjacency matrix, e.g., $G^{*} = ( \gV, \gE, \mA\odot \mM; \mX)$, $\odot$ means elements-wise multiplication. The mask highlights components of $G$ which are essential for $f$ to make the prediction. 

\subsection{Graph Information Bottleneck}
\label{sec:gib}
For a graph $G$ follows the distribution with $P_\mathcal{G}$, 
we aim to get an explainer function $\hat{G} = g_\alpha(G)$, where $\alpha$ are the parameters of explanation generator $g$. 
To solve the graph information bottleneck, previous methods~\citep{ying2019gnnexplainer, luo2020parameterized, zhang2023mixupexplainer}
are optimized under the following objective function: 
\begin{equation}
        G^* = \arg \min_{\hat{G}} I(G, \hat{G}) - \lambda I(Y, \hat{G}), \hat{G}\in \mathcal{\hat{G}}, 
    \label{eq: gib}
\end{equation}
where $G^*$ is the optimized sub-graph explanation produced by optimized $g$, $\hat G$ is the explanation candidate, and $\mathcal{\hat{G}} = \{\hat{G}\}$ is the set of $\hat{G}$ observations. During the explaining procedure, this objective function would minimize the size constraint $I(G, \hat G)$ and maximize the label mutual information $I(Y, \hat G)$. $\lambda$ is the hyper-parameter which controls the trade-off between two terms.

Since it is untractable to directly calculate the mutual information between prediction label $Y$ and sub-graph explanation $\hat G$, to estimate the objective, Eq.~(\ref{eq: gib}) is approximated as 
\begin{equation}
    G^* = \arg \min_{\hat{G}} I(G, \hat{G}) - \lambda I(Y, f(\hat{G})), \hat{G}\in \mathcal{\hat{G}}.
    \label{eq: sim_gib_1}
\end{equation}
Since $G^*$ is generated by $g_{\alpha}(\cdot)$, instead of optimize $\hat G$, we optimize $\alpha^*$ with 
$h(\alpha, \lambda, G, f) = I(G, g_{\alpha}(G)) - \lambda I(Y, f(g_{\alpha}(G)))$, then we have $\alpha^* = \arg \min h(\alpha, \lambda, G, f)$, where $I(Y, G^*) \sim I(Y, f(G^*))$ and $f$ is the pre-trained GNN model. 

\section{Methodology}
\label{sec:method}
Fig.~(\ref{fig:main}) presents an overview of the structure of \ours. The previous method is depicted on the left, while \ours incorporates a Bayesian Variational Inference process into the entire architecture, utilizing a large language model as the grader.
(In the previous explanation procedure, the explanation sub-graph is generated via an explanation model to explain the original graph and the to-be-explained prediction model. The explanation model is optimized by minimizing the size constraint $I(G, \hat{G})$ and maximizing the label mutual information $I(Y, \hat{Y})$, within $h(\alpha, \lambda, G, f)$, which is introduced in Section~\ref{sec:gib}. 

In our proposed framework, after generating the explanation sub-graph, we evaluate it through a Bayesian Variational Inference process, which is realized using a Large Language Model agent acting as a human expert. The enhanced explanation $\hat G$ is then produced using Eq.~(\ref{eq:fitting}). Note that, with the introduction of the Bayesian Variational Inference process, the previous distribution for $\alpha$ in $g_{\alpha}(\cdot)$ will shift to $\beta$ in $g'_{\beta}(\cdot)$. Finally, we optimize the explanation model with new $I(G, \hat G)$ and $I(Y, \hat Y)$ within $h(\beta, \lambda, G, f)$. 
We provide detailed formulas for the LLM-based Bayesian Variational Inference in Section~\ref{sec: BI} and a detailed prompt-building procedure in Section~\ref{sec: prompt}. Our training procedure is provided in Algorithm.~(\ref{alg: train}).
\begin{figure*}
\begin{center}
    \includegraphics[width=0.95\textwidth]{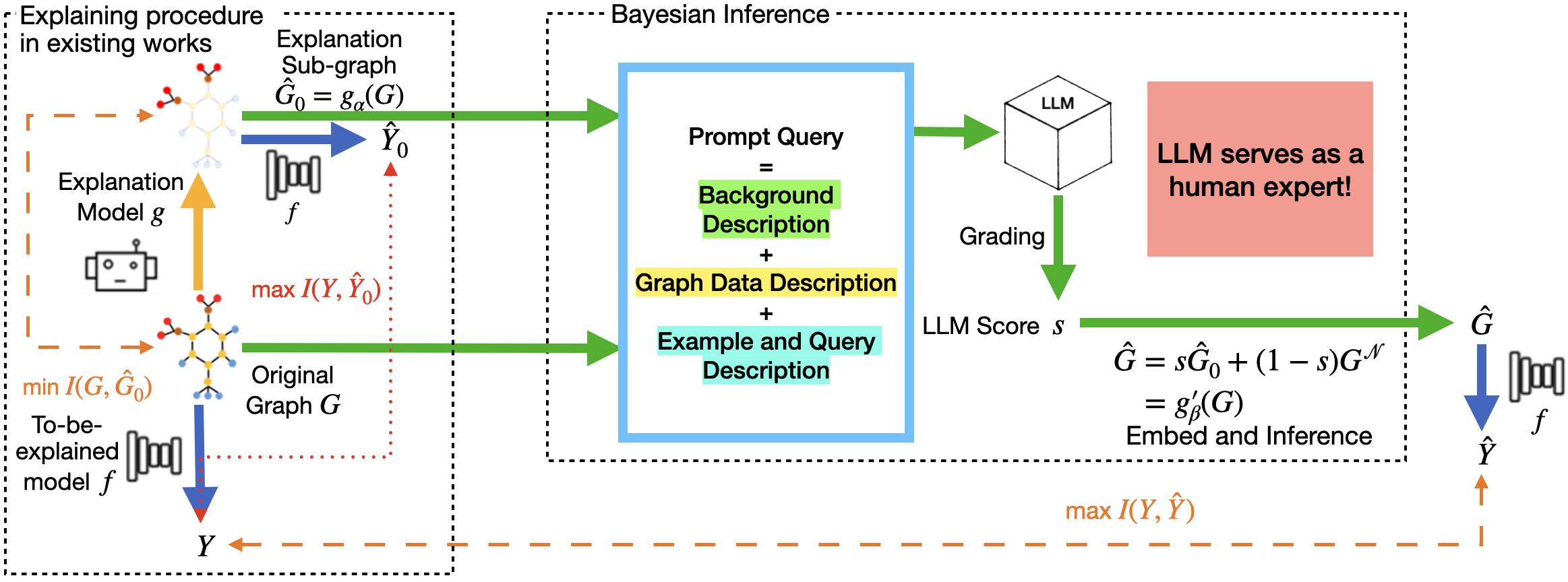}
    \caption{The architecture of our proposed framework for graph explaining and Bayesian inference on one graph sample in one epoch. 
    $\hat G_0$ is the sub-graph explanation candidate before Bayesian inference and $\hat Y_0 = f(\hat G_0)$ is the prediction label for $\hat G_0$. The explaining procedure in existing works on the left part optimizes the sub-graph explanation by minimizing the size constraint $I(G, \hat{G}_0)$  and maximizing the mutual information $I(Y, \hat{Y}_0)$, 
    which would face the learning bias problem after epochs. We introduced the Bayesian inference together with the LLM, which serves as a human expert to produce the embedded graph $\hat G$ and replace $I(Y, \hat Y_0)$ 
    with $I(Y, \hat Y)$in the objective function, where $\hat Y = f(\hat G)$ is the prediction label for $\hat G$. 
    The detailed illustration of prompting is shown in Fig.~(\ref{fig:prompt}). 
    } 
    \label{fig:main}
    \vspace{-0.6cm}
\end{center}
\end{figure*}









\subsection{Bayesian Variational Inference In Explainer}\label{sec: BI}
The injection of Bayesian Variational Inference into the GNN Explainer helps mitigate the learning bias problem. In this section, we will provide details on how we achieve this goal and present a theoretical proof for our approach. We begin by injecting the knowledge learned from the LLM as a weighted score and adding the remaining objective with a graph noise. Instead of adopting weighted noise, or gradient noise\citep{neelakantan2015adding}, we choose random Gaussian noise in this paper\citep{graves2011practical}.
\begin{problem} [Knowledge Enhanced Post-hoc Instance-level GNN Explanation]
\label{prob:exp}
Given a trained GNN model $f$ and a knowledgeable Large Language Model $\phi$, for an arbitrary input graph $G= (\mathcal{V}, \mathcal{E}; \mX, \mA)$, the goal of post-hoc instance-level GNN explanation is to find a sub-graph $G^{*}$ with Bayesian Inference embedded explanation generator $g'_{\beta}(\cdot)$, that can explain the prediction of $f$ on $G$ with the Large Language Model grading $s = \phi(\hat{G}, G)$, $s\in [0,1]$ in circle, as $\hat{G} = g'_{\beta}(G)= s g_{\alpha}(G) + (1-s ) G^\mathcal{N}$, $\beta^* = h(\beta, \lambda, G, f)$, where $\beta$ is the parameter for $g'(\cdot)$ and $g'_\beta$ is the explainer model integrated with the Bayesian inference procedure.  
\end{problem}
The distribution of $\alpha$ and $\beta$ will be different.
Suppose we add the fitting score as a posterior knowledge. Then we will have the prior probability as $P(\hat{G}|
\alpha)$, and the posterior probability as $Pr(\hat{G}|\alpha, s)$. With variational inference, suppose we approximates $Pr(\hat{G}|\alpha, s)$ over a class of tractable distributions $Q$, with $Q(\hat{G}|\beta)$, then the approximation becomes
\begin{equation}
    G^* = \arg \min _{\hat{G}}\mathbb{E}_ {\hat{G}\sim~ Q(\beta)}[-\ln\frac{Pr(s|\hat{G})P(\hat{G}|\alpha)}{Q(\hat{G}|\beta)}]
\end{equation}
We denote the network loss $L(\hat{G},s)$ as $L(\hat{G},s) = -\ln Pr(s|\hat{G})$ \citep{graves2011practical}, then we will have variational energy $F$ as 
\begin{equation}
    F = \mathbb{E}_ {\hat{G}\sim~ Q(\beta)}L(\hat{G},s) + KL[Q(\beta)||P(\alpha)] 
\end{equation}
\begin{definition}
    We define error loss $L^E(\beta, s)$ and complexity loss $L^C(\beta, s)$ in which 
    \begin{equation}
        \begin{aligned}
            L^E(\beta, s) &=\mathbb{E}_ {\hat{G}\sim~ Q(\beta)}L(\hat{G},s)\\
            L^C(\beta, \alpha) &= KL[Q(\beta)||P(\alpha)]
        \end{aligned}
        \label{eq: lelc}
    \end{equation}
    Then we will have
    \begin{equation}
        \begin{aligned}
            L(\alpha, \beta, s) &= L^E(\beta, s) + L^C(\beta, \alpha)
        \end{aligned}
        \label{eq: loss_together}
    \end{equation}
    where $L(\alpha, \beta, s)$ is the \textit{Minimum Description length} form of variational energy $F$~\citep{rissanen1978modeling}.
\end{definition}

\begin{figure*}
\begin{center}
    \includegraphics[width=0.95\textwidth]{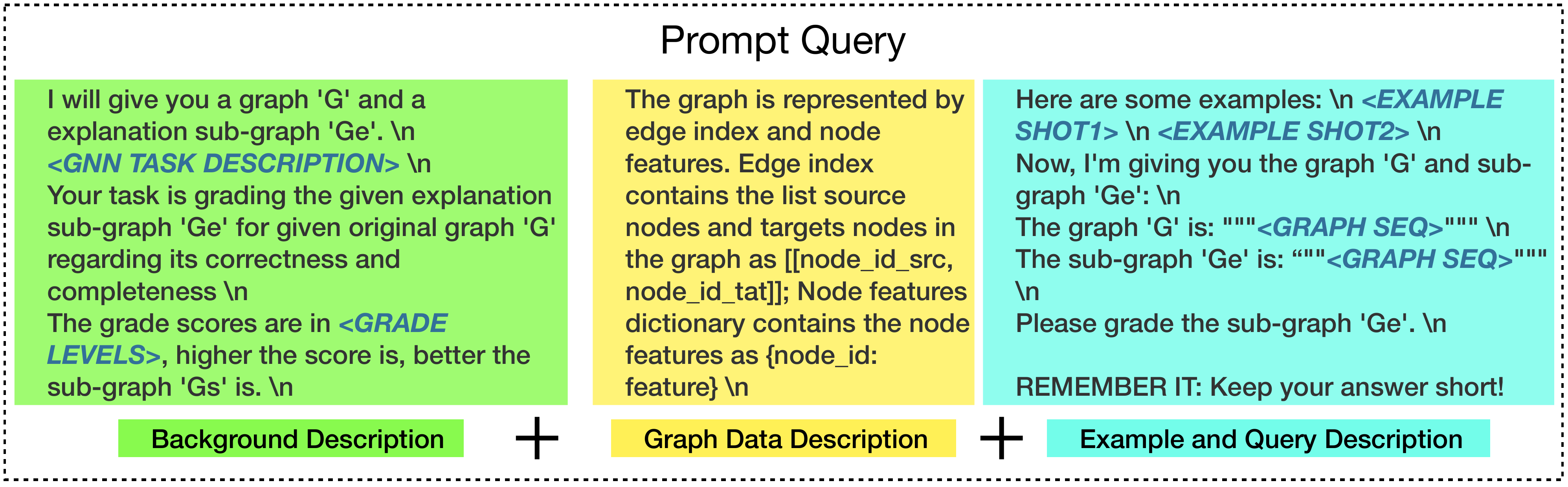}
    \scalebox{0.9}{
    \begin{tabular}{|l|l|}
    \hline
        Placeholder & Description \\
        \hline
        \textless GNN TASK DESCRIPTION\textgreater & A task-specific description to help LLM agent understand the graph task. \\
        \textless GRADE LEVELS\textgreater & A range of candidate scores for the LLM agent to choose from. \\
        \textless GRAPH SEQ\textgreater & A sequence consists of the edge index and node features of a graph sample. \\
        \textless EXAMPLE SHOT\textgreater & An example to teach LLM how to grade the explanation candidates. \\
        \hline
    \end{tabular}
    }
    \caption{The prompt construction of \ours in Bayesian Variational Inference Process. For an original graph $G$ and sub-graph explanation candidate $\hat G$, namely 'G' and 'Ge' in text, Large Language Models could tackle and reason the graph tasks with commonsense and expert knowledge, then grade the explanation candidate sub-graph. The placeholders would be replaced with the full text during prompting. Their usage is shown in the table.} 
    \label{fig:prompt}
    \vspace{-8mm}
\end{center}
\end{figure*}

\begin{hypothesis}
    We suppose that the $s$ is accurate enough to estimate the fitting between $\hat{G}$ and $G$. When $\hat{G}=G^*$, we have $s = 1$.
    \label{hypo: fitting_score}
\end{hypothesis} 
The hypothesis has been tested in  Fig.~(\ref{fig: RQ2}). In the following theorem, we prove that the objective function is at least sub-optimal to avoid the learning bias with accurate $s$.
\begin{theorem}
    When we reach the optimum $\hat{G}=G^*$, we will have the gradient $\Delta = \frac{\partial F}{\partial \hat{G^*}} \approx 0$, trapping $G^*$ at the optimum point to avoid learning bias.
    \label{the:opt}
\end{theorem}
\begin{proof}

The generation of $\hat{G}$ included two steps:
\begin{enumerate}
    \item Generate from the original generator $g_{\alpha}(\cdot)$ with $\hat{G}_0 = g_{\alpha}(G)$.
    \item Embed the fitting score with $\hat{G} = s\hat{G}_0 + (1-s 
        ) G^\mathcal{N},G^\mathcal{N} \sim \mathcal{N}(0, 1)$
    where $\hat{G}$ is the explanation sub-graph, $s$ is the LLM score of $\hat{G}$ with $s = \phi(\hat{G}, G)$, $G^\mathcal{N}$ is the Gaussian noise.
\end{enumerate}
Then we have the embedded inference network as
\begin{equation}
\begin{aligned}
    \hat{G} &= g'_{\beta}(G)= s g_{\alpha}(G) + (1-s ) G^\mathcal{N}
\end{aligned}
\label{eq:fitting}
\end{equation}
The distribution of $\hat{G}$ is denoted as $Q$, where $\hat{G}\sim Q(\beta)$. Then we can calculate $L^E(\beta, s)$ and $L^C(\beta, \alpha)$ in Eq.~(\ref{eq: lelc}) separately.
We will discuss the scenarios of error loss $L^E(\beta, s)$ and network loss $L^C(\beta, \alpha)$ when we have $\hat{G} = G^*$ and $s\to 1$.
\begin{equation}
    \begin{aligned}
        L^E(\beta, s) &=\mathbb{E}_ {\hat{G}\sim~ Q(\beta)}L(\hat{G},s)\\
        &= -\int P(\hat{G}|\beta)\cdot \ln Pr(s|\hat{G})d\hat{G}.\\
    \end{aligned}
    \label{eq: error_loss}
\end{equation}
When $s\to 1$, we have $L^E(\beta, s)\stackrel{s\to 1}{\approx}-\int sP(\hat{G_0}|\alpha)\cdot \ln Pr\left(s|\hat{G_0}\right)d\hat{G_0}$.

\begin{equation} 
    \begin{aligned}
        L^C(\beta, \alpha) &\sim KL[sg_{\alpha}(G) + (1-s) G^\mathcal{N}|| g_{\alpha}(G)]\\
        &=\int (s\hat{G}_0 + (1-s)G^\mathcal{N}) \\ 
        &\quad \log\left(\frac{s\hat{G}_0 + (1-s)G^\mathcal{N}}{\hat{G}_0}\right) d\hat{G}_0.\\
    \end{aligned}
    \label{eq: network_loss}
\end{equation}
When $s\to 1, \frac{1-s}{s}\frac{G^\mathcal{N}}{\hat{G}_0} \to 0$, we have $L^C(\beta, \alpha)\stackrel{s\to 1}{\approx} (1-s)(G^\mathcal{N}-\hat{G}_0)d\hat{G}_0$.

With Eq.~(\ref{eq: loss_together}), when $\hat{G_0}$ is optimal and $s\to 1$, we have 
\begin{equation}
    \begin{aligned}
        \Delta &=\frac{\partial F}{\partial \hat{G_0}} = -sP(\hat{G_0}|\alpha)\cdot \ln Pr\left(s|\hat{G_0}\right) + \\
                & \quad(1-s)(G^\mathcal{N}-\hat{G}_0)\\
    \end{aligned}
\end{equation}
Based on Hypothesis \ref{hypo: fitting_score}, we have $Pr(s|\hat{G_0})\to 1$, $\Delta \approx 0$.
Hence, when we reach the optimum sub-graph $G^*$, our algorithm will trap the optimum with the gradient $\Delta \to 0$ to mitigate learning bias. Due to page limitations, we present the full proof in the Appendix.
\end{proof}

\subsection{Prompting}
\label{sec: prompt}
As shown in Fig.~(\ref{fig:prompt}), we build a pair $(G, \hat G)$ into a prompt. It contains several parts: (1). We provide a background description for the task. \textbf{\textless GNN TASK DESCRIPTION\textgreater} is a task-specific description. For example: for dataset {\bamc}, it would be "The ground truth explanation sub-graph 'Ge' of the original graph 'G' is a circle motif."; for chemical dataset {\mutag}, it would be "The ground truth explanation sub-graph 'Ge' of the original graph 'G' is sub-compound and decides the molecular property of mutagenicity on Salmonella Typhimurium." \textbf{\textless GRADE LEVELS\textgreater} is a range of candidate scores according to the task, e.g.: $[0, 1]$. (2). We describe how we express a graph in text, which would help the LLM understand our graph sample. The \textbf{\textless GRAPH SEQ\textgreater} contains two parts: edge index and node features, which transport the graph-structure data into text, e.g.:edge index: $[[0,1], [1, 2], [2, 0]]$, node feature: $\{0: [3.6], 1: [2.4], 2:[9.9]\}$. (3). We then provide a few shots to the LLM. The \textbf{\textless EXAMPLE SHOT\textgreater} contains several pairs of to-be-grade candidates and their corresponding grades. We ask the LLM to grade our query candidate. (4). Finally, we regularize the answer of LLM to be a single number with the "REMEMBER IT: Keep your answer short!" prompt. We put the complete samples in the GitHub repository along with our code and data.

\section{Experimental Study}
\begin{table*}
  \setlength{\tabcolsep}{4.5pt}
  \vspace{-5mm}
  \caption{Explanation faithfulness in terms of AUC-ROC on edges under five datasets. The higher, the better. Our proposed method achieves consistent improvements over the backbone GIB-based explanation method and other baselines.} 
  \vspace{-2mm}
  \scalebox{0.9}{
  \begin{tabular}{c|ccc|cc}
  \hline
                    & \multicolumn{3}{c|}{Graph Classification}                                   &   \multicolumn{2}{c}{Graph Regression} \\
    \hline
                    & \mutag            & \fluc                 & \alca                 & \bamv                 & \bamc\\
    \hline
    GRAD            &$0.573_{\pm 0.000}$& $0.604_{\pm 0.000}$   & $0.496_{\pm 0.000}$   & $0.418_{\pm 0.000}$   & $0.516_{\pm 0.000}$\\
    ReFine          &$0.612_{\pm 0.004}$& $0.507_{\pm 0.003}$   & $0.630_{\pm 0.007}$   & $0.572_{\pm 0.006}$   & $0.742_{\pm 0.011}$ \\
    GNNExplainer    &$0.682_{\pm 0.009}$& $0.540_{\pm 0.002}$   &$0.532_{\pm 0.008}$    & $0.501_{\pm 0.009}$   & $0.504_{\pm 0.003}$ \\
    \hline
    PGExplainer     &$0.601_{\pm 0.151}$&   $0.512_{\pm 0.010}$ &  $0.586_{\pm 0.031}$  & $0.547_{\pm 0.040}$   & $0.969_{\pm 0.017}$\\
    + LLM           &$\mathbf{0.763_{\pm 0.106}}$&$\mathbf{0.719_{\pm 0.017}}$ &$\mathbf{0.791_{\pm 0.003}}$&$\mathbf{0.979_{\pm 0.009}}$&$\mathbf{0.990_{\pm 0.004}}$\\ 
    (improvement)   &      $+0.162$     &   $+0.207$            &  $+0.205$             & $+0.432$              & $+0.021$ \\
    \hline
  \end{tabular}
  }
  \vspace{-5mm}
  \label{tab:QEtable}
\end{table*}

We conduct comprehensive experimental studies on benchmark datasets to empirically verify the effectiveness of the proposed \ours. Specifically, we aim to answer the following research questions:

\begin{itemize}
    \item RQ1:   Could the proposed framework outperform the baselines in identifying the explanation sub-graphs for the to-be-explained GNN model?
    \item RQ2:   Could the LLM score help address the learning bias issue? 
    \item RQ3:   Could the LLM score reflect the performance of the explanation sub-graph; is it effective in the proposed method? 
\end{itemize}
\vspace{-2mm}

\subsection{Experiment Settings}
To evaluate the performance of {\ours}, we use five benchmark datasets with ground-truth explanations. These include two synthetic graph regression datasets: {\bamv} and {\bamc}~\citep{zhang2023regexplainer}, with three real-world datasets: {\mutag} \citep{kazius2005derivation}, {\fluc}~\citep{sanchez2020evaluating}, and {\alca} \citep{sanchez2020evaluating}. We take GRAD~\citep{ying2019gnnexplainer}, GNNExplainer~\citep{ying2019gnnexplainer}, ReFine~\citep{wang2021towards}, and PGExplainer~\citep{luo2020parameterized} for comparison. Specifically, we pick PGExplainer as a backbone and apply \ours to it. We follow the experimental setting in previous works \citep{ying2019gnnexplainer, luo2020parameterized, sanchez2020evaluating, zhang2023mixupexplainer} to train a Graph Convolutional Network (GCN) model with three layers. We use GPT-3.5 as our LLM grader for the explanation candidates.
To evaluate the quality of explanations, we approach the explanation task as a binary classification of edges. Edges that are part of ground truth sub-graphs are labeled as positive, while all others are deemed negative. We take the importance weights given by the explanation methods as prediction scores. An effective explanation technique should be able to assign higher weights to the edges within the ground truth sub-graphs compared to those outside of them. We utilize the AUC-ROC metric($\uparrow$) for quantitative evaluation. \footnote{Our data and code are available at:
\url{https://anonymous.4open.science/r/LLMExplainer-A4A4}}. 
\subsection{Quantitative Evaluation (RQ1)}

In this section, we compare our proposed method, {\ours}, to other baselines. Each experiment was conducted 10 times using random seeds from 0 to 9, with 100 epochs, and the average AUC scores as well as standard deviations are presented in Table~\ref{tab:QEtable}.
The results demonstrate that {\ours} provides the most accurate explanations among several baselines. Specifically, it improves the AUC scores by an average of $0.227/42.5\%$ on synthetic datasets and $0.191/34.1\%$ on real-world datasets. In dataset {\bamv}, we achieve $0.432/79.0\%$ improvement compared to the PGExplainer baseline. The reason is there is serious learning bias with PGExplainer on {\bamv}, which is shown in Fig.~(\ref{fig: RQ2}). The performance improvement in {\bamv} is not significant because of (1). the learning bias in this dataset is specifically slight; (2). The PGExplainer is well-trained at the epoch 100. Comparisons with baseline methods highlight the advantage of Bayesian inference and LLM serving as a knowledge agent in training.


\subsection{Qualitative Evaluation (RQ2)}
\begin{figure*}
    \vspace{-5mm}
    \centering
    \scalebox{0.95}{
    \begin{tabular}{ccccc}
        \subfigure{\includegraphics[width=0.19\textwidth]{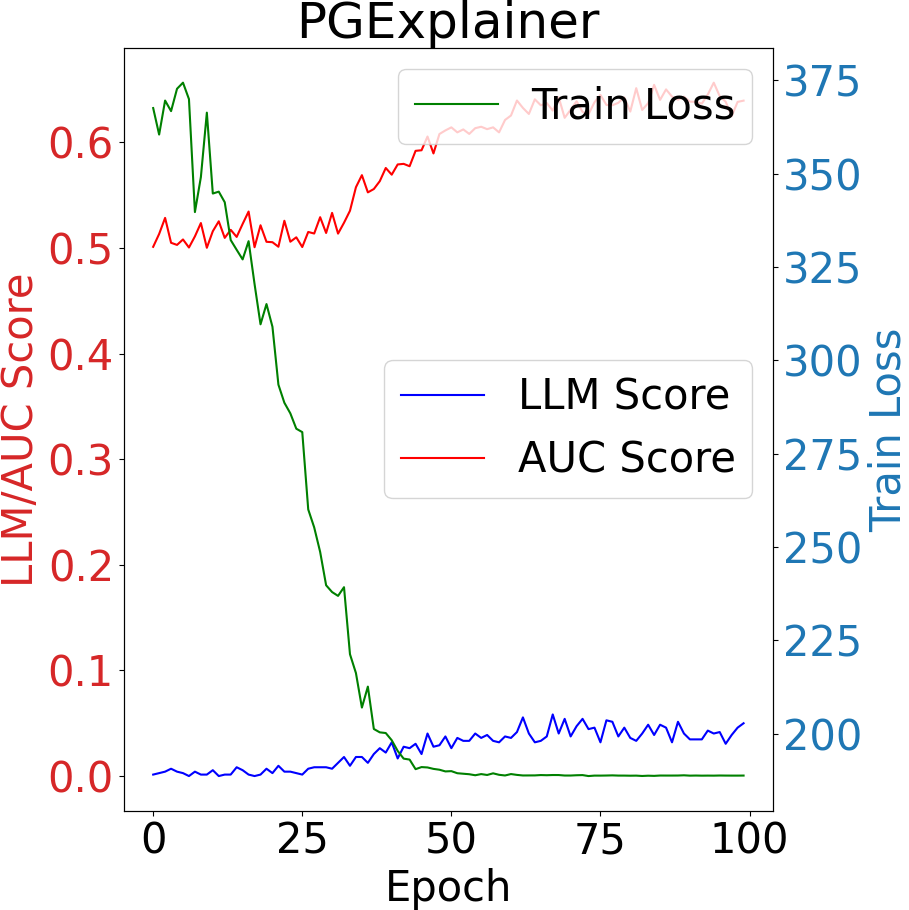}}  & 
        \subfigure{\includegraphics[width=0.19\textwidth]{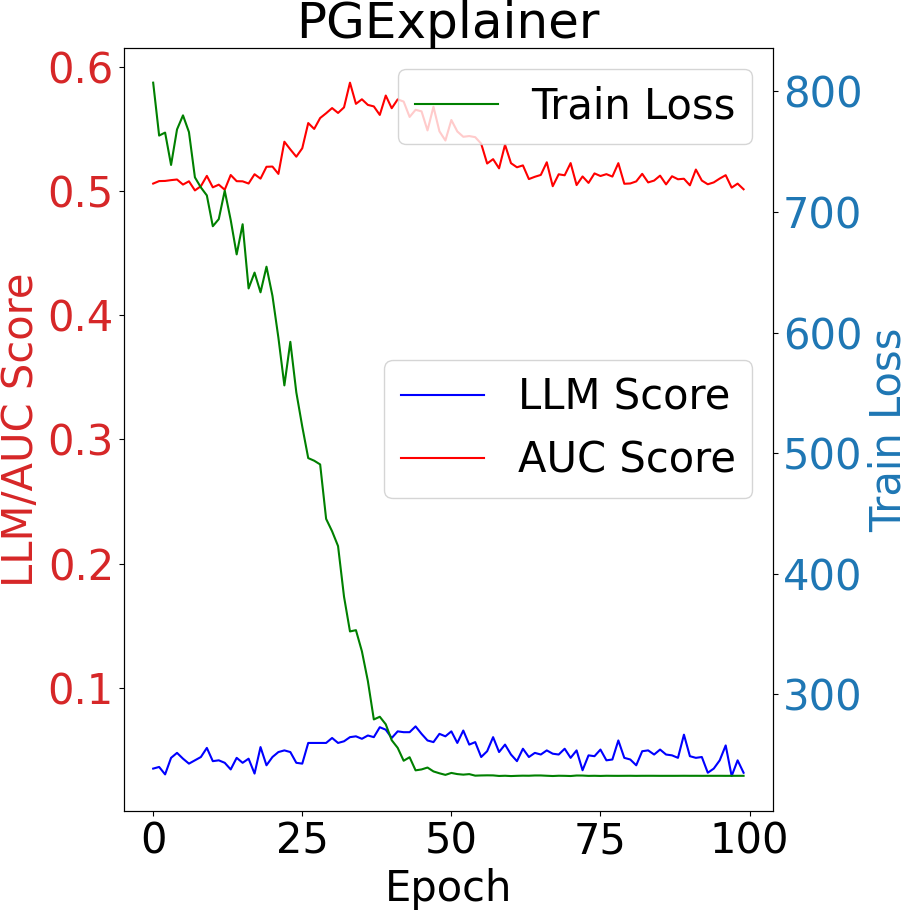}}   & 
        \subfigure{\includegraphics[width=0.19\textwidth]{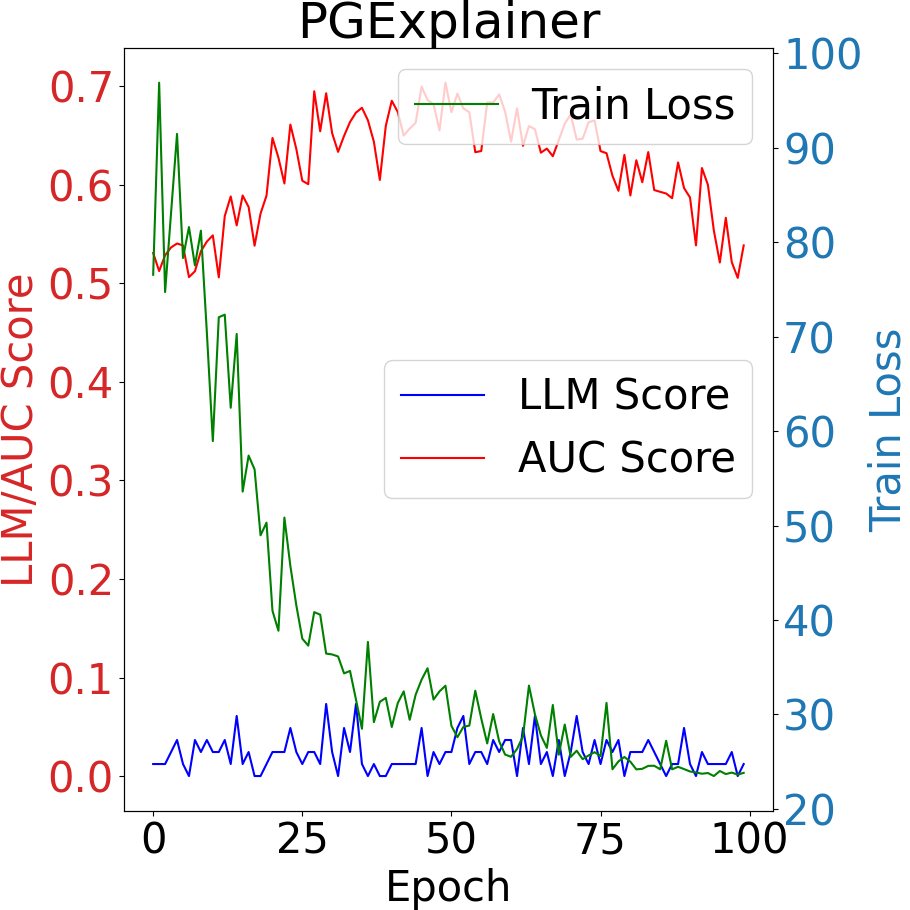}}  & 
        \subfigure{\includegraphics[width=0.19\textwidth]{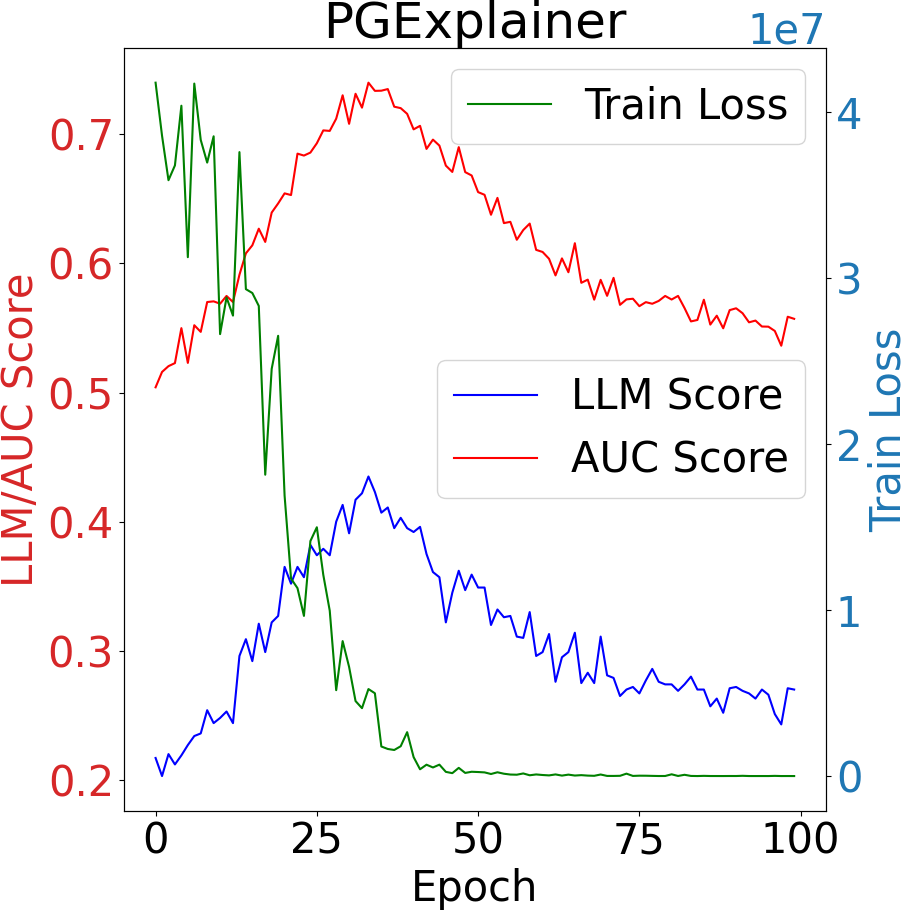}} & 
        \subfigure{\includegraphics[width=0.19\textwidth]{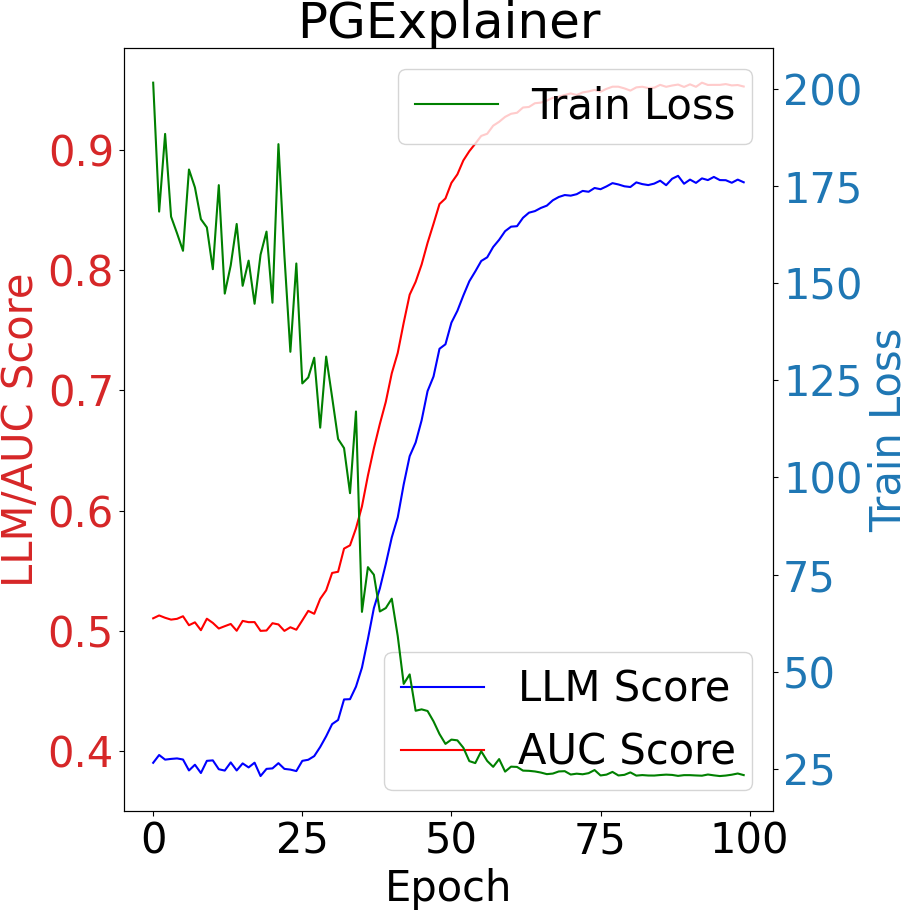}} \\
        \subfigure{\includegraphics[width=0.19\textwidth]{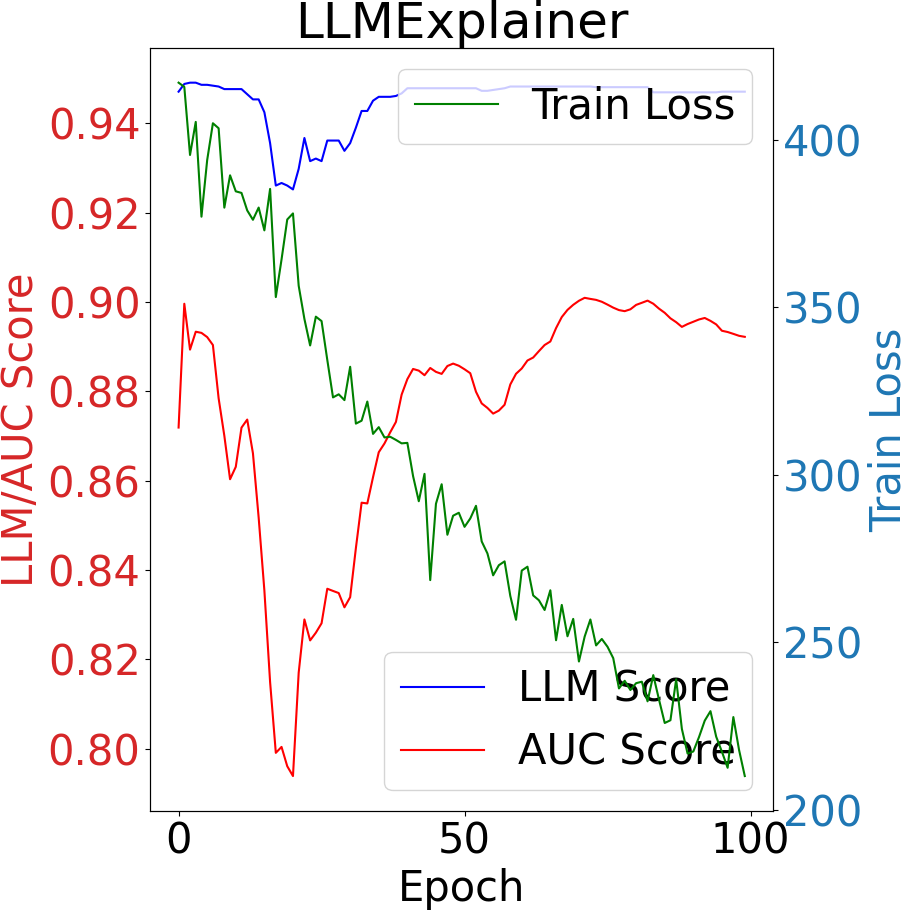}} &
        \subfigure{\includegraphics[width=0.19\textwidth]{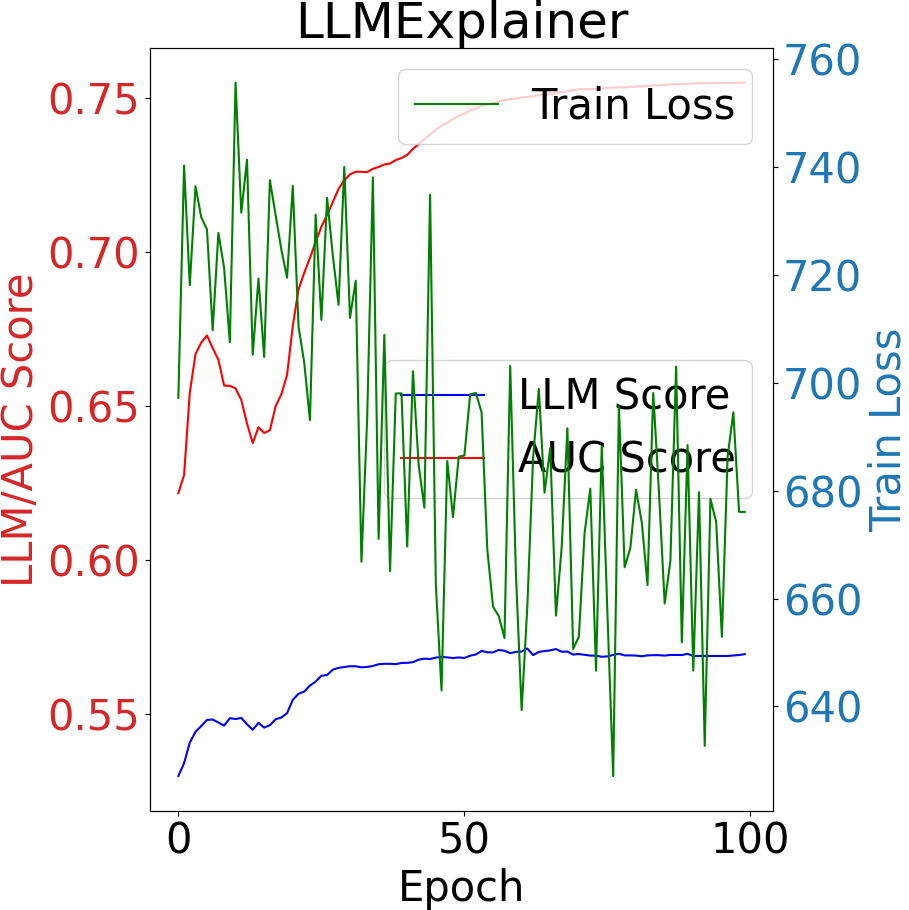}}  &
        \subfigure{\includegraphics[width=0.19\textwidth]{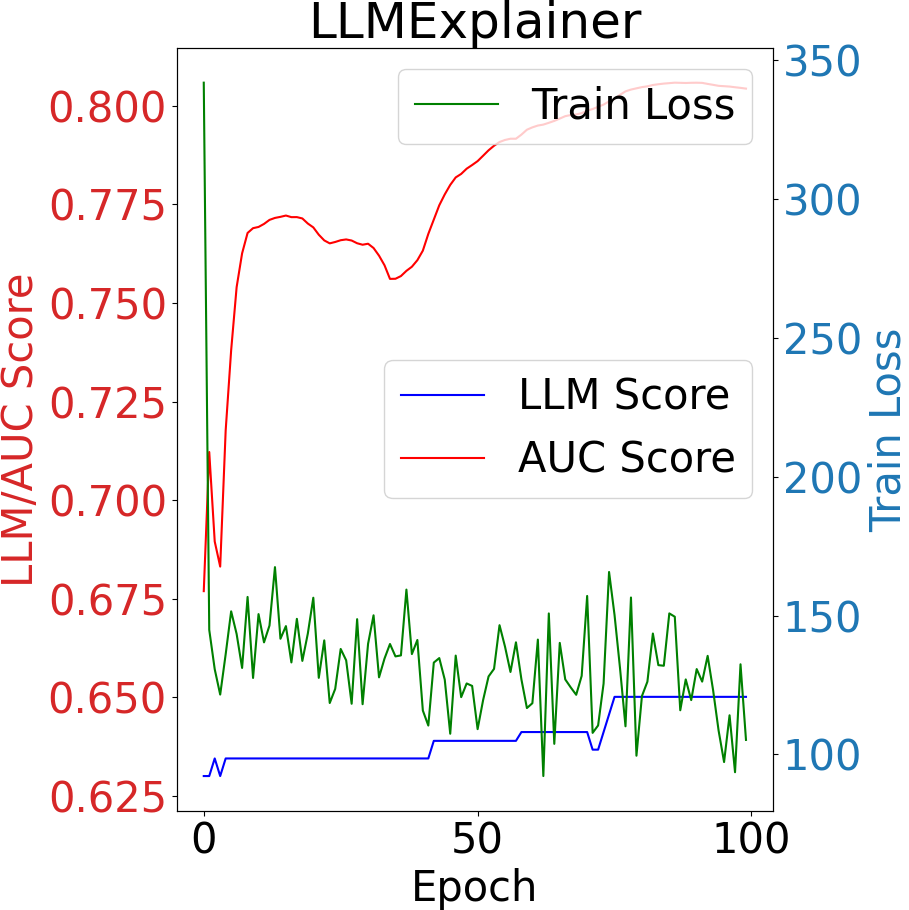}} &
        \subfigure{\includegraphics[width=0.19\textwidth]{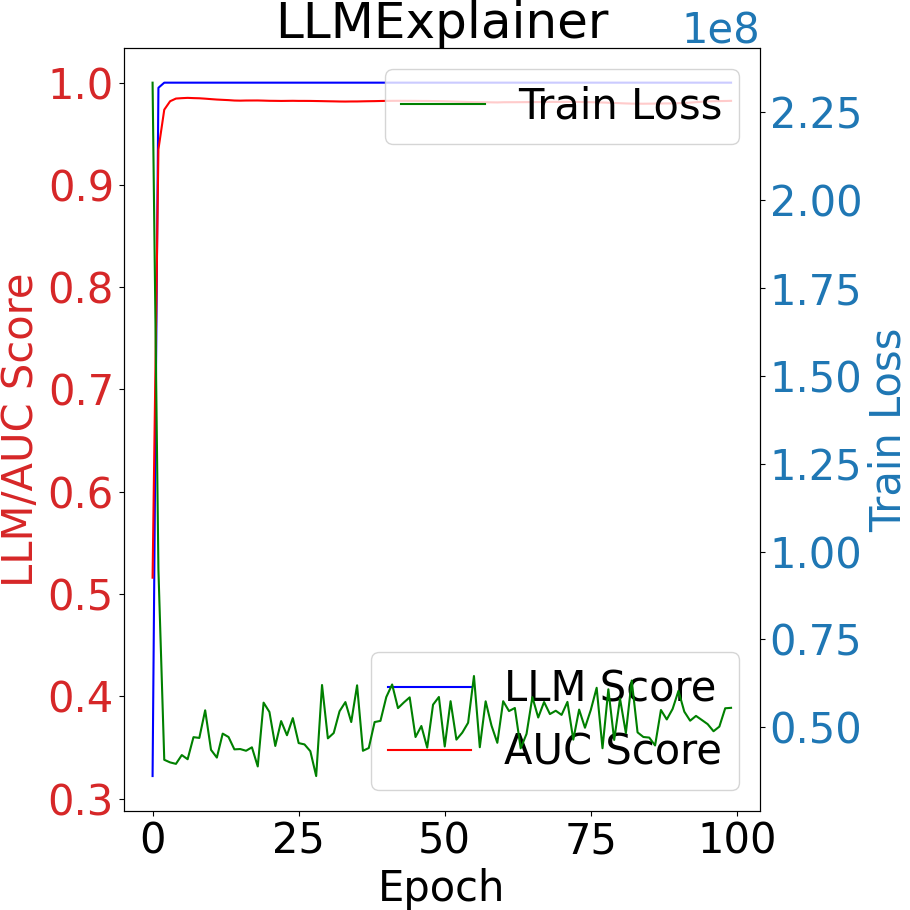}} &
        \subfigure{\includegraphics[width=0.19\textwidth]{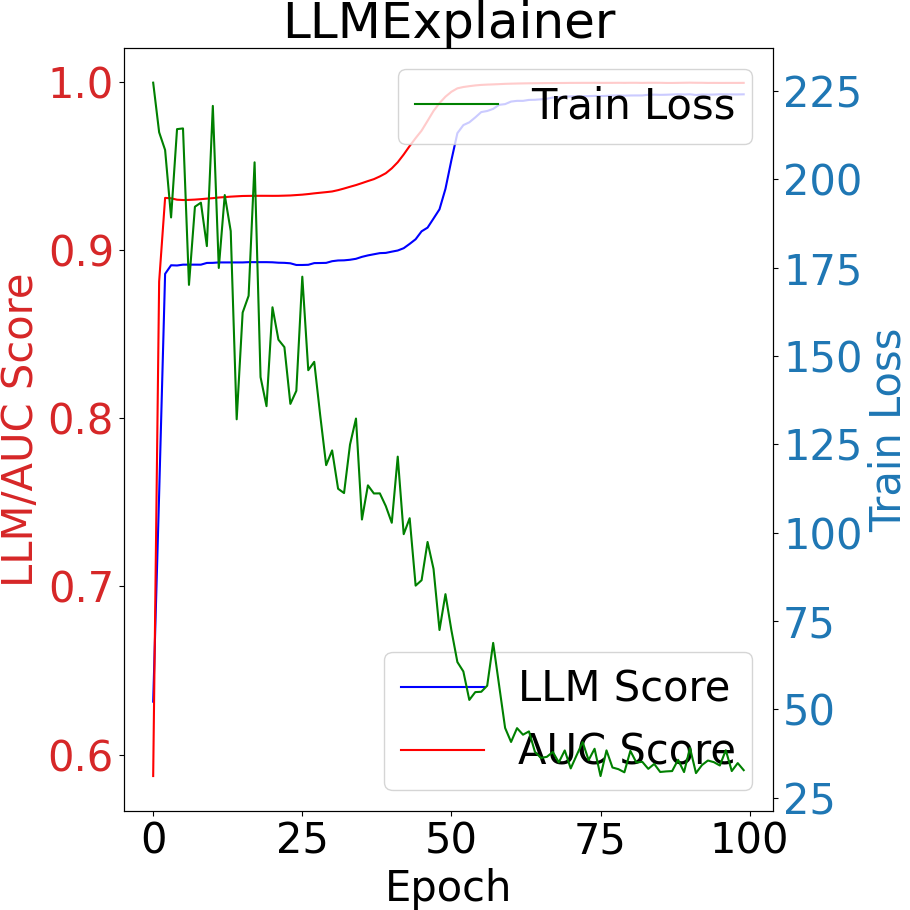}} \\
        {\tiny (a)~\mutag} & {\tiny (b)~\fluc} & {\tiny (c)~\alca} & {\tiny (d)~\bamv} & {\tiny (e)~\bamc} \\
    \end{tabular}
    }
    \vspace{-1mm}
    \caption{Visualization of AUC, LLM Score, and Training Loss curve by training epochs. We choose one seed for each explainer on each dataset.
    }
    \vspace{-1mm}
    \label{fig: RQ2}
\end{figure*}

\begin{table*}
  \vspace{-0mm}
  \label{tab:expl_auroc} 
  \scalebox{0.88}{
  \begin{tabular}{|c|ccccc|}
  \hline
                    & \mutag            & \fluc                 & \alca                 & \bamv                 & \bamc\\
    \hline
    Random Score    &$0.456_{\pm 0.280}$& $0.601_{\pm 0.058}$   & $0.683_{\pm 0.092}$   & $0.798_{\pm 0.111}$   & $0.887_{\pm 0.146}$\\
    LLM Score       &$0.763_{\pm 0.106}$& $0.719_{\pm 0.017}$   & $0.791_{\pm 0.003}$   & $0.979_{\pm 0.009}$   & $0.990_{\pm 0.004}$\\
    \hline
  \end{tabular}
  }
  \vspace{-0mm}
  \caption{Ablation study on the effectiveness of the LLM score. The results in the table compare the performance of the proposed method with the LLM score and random score.} 
  \vspace{-2mm}
  \label{tab: RQ3}
\end{table*}

As shown in Fig.~(\ref{fig: RQ2}), we visualize the training procedure of PGExplainer and {\ours} on five datasets. The first row is PGExplainer and the second row is {\ours}. From left to right, the five datasets are {\mutag}, {\fluc}, {\alca}, {\bamv}, and {\bamc} respectively. We visualize and compare the AUC performance, LLM score, and train loss of them. Additionally, the LLM score is not used in PGExplainer, we retrieve it with the explanation candidates during training. As we can observe in the {\bamv} dataset, the AUC performance and LLM score increase in the first 40 epochs. However, they drop persistently to around 0.5/0.3 after 100 epochs, indicating a learning bias problem. For LLMExplainer, based on PGExplainer, on the second row, we can observe that the AUC performance and LLM Score increase to 0.9+ and maintain themselves stably. This observation is similar in {\fluc} and {\alca} datasets. The LLM score slightly dropped at around epoch 40 and then recovered. The results show that our proposed framework {\ours} could effectively alleviate the learning bias problem during the explaining procedure. 
\subsection{Ablation Study (RQ3)}

In this section, we conduct the experiments on the ablation study of the {\ours}. When we have the LLM score, we may be concerned about whether or not this score could reflect the real performance of the explanation sub-graph candidates and whether it would work for the proposed framework. So, we compare the performance of the framework with the LLM score and random score. In the model with the random score, we replace the LLM grading module with a random number generator. As shown in Table~{\ref{tab: RQ3}}, the results show that the LLM score could effectively enhance the Bayesian Inference procedure in explaining.

\vspace{-2mm}
\section{Conclusion}
In this work, we proposed LLMExplainer, which incorporated the graph explaining the procedure with Bayesian Inference and the Large Language Model. We build the to-be-explained graph and the explanation candidate into the prompt and take advantage of the Large Language Model to grade the candidate. Then we use this score in the Bayesian Inference procedure. We conduct experiments on three real-world graph classification tasks and two synthetic graph regression tasks to demonstrate the effectiveness of our framework. By comparing the baselines and PGExplainer as the backbone, the results show the advantage of our proposed framework. Further experiments show that the LLM score is effective in the whole pipeline.

\section{Limitations}

While we acknowledge the effectiveness of our method, we also recognize its limitations. Specifically, although our approach theoretically and empirically proves the benefits of integrating the LLM agent into the explainer framework, the scope of the datasets and Large Language Models remains limited. To address this challenge, we plan to explore more real-world datasets and deploy additional LLMs for a more comprehensive evaluation in the future work.
\bibliography{colm2024_conference}

\clearpage
\section{Appendix}

\subsection{Error loss $L^E(\beta, s)$}
The full proof of Equation \ref{eq: error_loss} towards the error loss $L^E(\beta, s)$ is
\begin{equation}
    \begin{aligned}
        L^E(\beta, s) &=\mathbb{E}_ {\hat{G}\sim~ Q(\beta)}L(\hat{G},s)\\
        &= -\int P(\hat{G}|\beta)\cdot \ln Pr(s|\hat{G})d\hat{G}\\
        &\text{When $s\to 1$, then we have}\\
        &\stackrel{s\to 1}{\approx} -\int P(\hat{G}|\alpha)\cdot \ln Pr(s|\hat{G})d\hat{G}\\
        &=-\int sP(\hat{G_0}|\alpha)\cdot \\
        & \quad \ln Pr\left(s| s\hat{G_0}+(1-s)G^\mathcal{N}\right)d\hat{G_0}\\
        &=-\int sP(\hat{G_0}|\alpha)\cdot \ln Pr\left(s|\hat{G_0}\right)d\hat{G_0}\\
    \end{aligned}
\end{equation}
\subsection{Network loss $L^C(\beta, \alpha)$} 

The full proof of Equation \ref{eq: network_loss} towards the network loss $L^C(\beta, \alpha)$ is
\begin{equation} 
    \begin{aligned}
        L^C(\beta, \alpha) &\sim KL[sg_{\alpha}(G) + (1-s) G^\mathcal{N}|| g_{\alpha}(G)]\\
        &=\int (s\hat{G}_0 + (1-s)G^\mathcal{N}) \log\left(\frac{s\hat{G}_0 + (1-s)G^\mathcal{N}}{\hat{G}_0}\right) d\hat{G}_0\\
        &=\int (s\hat{G}_0 + (1-s)G^\mathcal{N}) \log\left(s + (1-s)\frac{G^\mathcal{N}}{\hat{G}_0}\right) d\hat{G}_0\\
        &=\int (s\hat{G}_0 + (1-s)G^\mathcal{N}) \log s\left(1 + \frac{1-s}{s}\frac{G^\mathcal{N}}{\hat{G}_0}\right) d\hat{G}_0\\
        &=\int (s\hat{G}_0 + (1-s)G^\mathcal{N}) \left(\log s + \log\left(1 + \frac{1-s}{s}\frac{G^\mathcal{N}}{\hat{G}_0}\right)\right) d\hat{G}_0\\
        &\text{When $s\to 1, \frac{1-s}{s}\frac{\hat{G}_0}{G^\mathcal{N}} \to 0$, with Taylor Series, then we have}\\ 
        &\stackrel{s\to 1}{\approx}\int (s\hat{G}_0 + (1-s)G^\mathcal{N}) \left(\log s + \frac{1-s}{s}\frac{G^\mathcal{N}}{\hat{G}_0}\right) d\hat{G}_0\\
        &=\int s\hat{G}_0\left(\log s + \frac{1-s}{s}\frac{G^\mathcal{N}}{\hat{G}_0}\right) d\hat{G}_0\\
        &=\int s\hat{G}_0\log s + (1-s)G^\mathcal{N}d\hat{G}_0\\
        &=\int s\hat{G}_0\log s + (1-s)G^\mathcal{N}d\hat{G}_0\\
        &=\int s\hat{G}_0\log (1- (1-s)) + (1-s)G^\mathcal{N}d\hat{G}_0\\
        &=\int -(1-s)s\hat{G}_0 + (1-s)G^\mathcal{N}d\hat{G}_0\\
        &\approx \int (1-s)(G^\mathcal{N}-\hat{G}_0)d\hat{G}_0\\
    \end{aligned}
\end{equation}

\subsection{Symbol Table}
\begin{table*}
  \label{tab: symbols} 
  \begin{center}
  \scalebox{0.9}{
  \begin{tabular}{|c|p{10cm}|}
  \hline
        Symbol Name & Symbol Meaning\\
        \hline
        $G$ & original to-be-explained graph  \\
        $Y$ & prediction label for $G$ \\
        $G^*$ & optimized sub-graph explanation  \\
        $Y^*$ & prediction label for $G^*$  \\
        $I(\cdot)$ & Mutual Information  \\
        $\mathcal{V}$ & Node set  \\
        $\mathcal{E}$ & Edge set\\
        $\mX$ & Feature matrix  \\
        $\mA$ & Adjacency matrix  \\
        $n$ & Number of nodes  \\
        $d$ & Dimension of feature  \\
        $i$ & The $i$-th node   \\
        $j$ & The $j$-th node  \\
        $A^{ij}$ & edge from node i to node j, not used later \\
        $\Bar{Y}$ & ground-truth label for graph $G$ \\
        $f$ & to-be-explained GNN model \\
        $\hat G$ &  The sub-graph generated by explainer \\
        $\hat{G}_0$ & The sub-graph generated by original explainer  \\
        $g_\alpha$ & original explainer  \\
        $g'_\beta$ &  GNN explainer with LLM embedded \\
        $\alpha^*$&  The parameters of original explainer $g_\alpha$ \\
        $\beta^*$ & The parameters of LLM embedded explainer $g'_\beta$  \\
        $\phi$ & Large Language Model  \\
        $\lambda$ & Hyper-parameter for the trade-off between size constraint and mutual information  \\
        $\hat{\mathcal{G}}$ & Set of $\hat G$  \\
        
        $H$ &  The loss function for GNN explainer \\
       
        $\hat{Y}_0$ &  Prediction Label for $\hat G_0$ \\
        $\hat Y$ & Prediction Label for $\hat G$ \\
        $s$ & LLM score, fitting score, grading   \\
        
        $G^\mathcal{N}$ &  The random noise graph with $G^\mathcal{N}\sim \mathcal{N}(0,1)$ \\
        
        $Q$ & The distribution of $\beta$  \\
        $F$ &  Variational energy, proposed by \citet{graves2011practical} \\
        $L(\hat G, s)$&  Network loss, proposed by \citet{graves2011practical} \\
        $L^E(\beta, s)$ &   Error loss, proposed by \citet{graves2011practical}\\
        $L^C(\beta, s)$&  Complexity loss, proposed by \citet{graves2011practical} \\
        $KL(\cdot)$ & KL divergence  \\
        $L(f, \beta, s)$ &  Minimum description length form of variational energy $F$\citep{rissanen1978modeling, graves2011practical} \\
        $\Delta$&  The gradient of $F$ toward $\hat{G_0}$ with $\Delta =\frac{\partial F}{\partial \hat{G_0}}$ \\
        \hline
  \end{tabular}
  }
  \caption{Important notations and symbols table.
  } 
  \label{app:tab:notation}
  \end{center}
\end{table*}

\subsection{Training Algorithm}

\renewcommand{\algorithmicrequire}{\textbf{Input:}}
\renewcommand{\algorithmicensure}{\textbf{Output:}}

\begin{algorithm}
	\caption{Training Algorithm} 
	\begin{algorithmic}[1]
        \Require Graph dataset $\gG$, explanation generator $g$. 
        \Ensure Trained explanation generator $g_\alpha$.
        \State Initialize explanation generator $g_{\alpha^{(0)}}$.
        \For {$t \in epochs$}
            \For{$G \in \mathcal{G}$}
                \State $\hat{G}_0^{(t)} \gets g_{\alpha^{(t)}}(G)$
                \State Prompt Query $q_t \gets $ prompting $(G, \hat{G}_0^{(t)})$ into sequence
                \State $\vs_t \gets \phi(q_t)$
                \State $G^\mathcal{N} \gets $ randomly sampled graph noise
                \State $\hat{G}^{(t)} \gets \vs_t \hat{G}_0^{(t)} + (1 - \vs_t) G^\mathcal{N}$
                \State Compute ${L}_{\text{GIB}}$
            \EndFor
            \State Update $g_{\alpha^{(t)}}$ with loss back propagation.
        \EndFor
        \State \Return Trained explanation generator $g_\alpha$.
    \end{algorithmic} 
    \label{alg: train}
\end{algorithm}

\end{document}